\documentclass[10pt]{article} 

\usepackage[preprint]{tmlr}

\usepackage{amsmath,amsfonts,bm}

\def\eqref#1{equation~\ref{#1}}

\def\1{\bm{1}}

\DeclareMathAlphabet{\mathsfit}{\encodingdefault}{\sfdefault}{m}{sl}
\SetMathAlphabet{\mathsfit}{bold}{\encodingdefault}{\sfdefault}{bx}{n}

\newcommand{\E}{\mathbb{E}}

\newcommand{\R}{\mathbb{R}}

\newcommand{\KL}{D_{\mathrm{KL}}}

\usepackage{hyperref}
\usepackage{url}

\usepackage{amsthm}

\usepackage{amsmath}
\DeclareMathOperator{\sgn}{sgn}
\DeclareMathOperator{\Bern}{Bern}

\theoremstyle{plain}
\newtheorem{theorem}{Theorem}
\newtheorem{lemma}[theorem]{Lemma}

\newtheorem{corollary}[theorem]{Corollary}

\theoremstyle{definition}

\newtheorem{assumption}[theorem]{Assumption}

\theoremstyle{remark}
\newtheorem{remark}[theorem]{Remark}

\usepackage{amssymb}

\providecommand{\X}{\mathcal{X}}
\providecommand{\Y}{\mathcal{Y}}
\providecommand{\R}{\mathbb{R}}
\providecommand{\N}{\mathbb{N}}
\providecommand{\E}{\mathbb{E}}
\providecommand{\W}{\mathbf{W}}

\renewcommand{\L}{\mathcal{L}}
\renewcommand{\v}{\mathbf{v}}

\usepackage{xcolor}
\newcommand{\Joel}[1]{\textcolor{blue}{[Joel: #1]}}

\title{Statistical Learning Theory for Direct Preference Optimization}

\author{\name Joel Wills \email J.Wills@maths.usyd.edu.au \\
      \addr School of Mathematics and Statistics \\ The University of Sydney
      \AND
      \name Yuanfan Li \email yuli0846@uni.sydney.edu.au \\
      \addr School of Mathematics and Statistics \\ The University of Sydney
      \AND
      \name Yiming Ying \email yiming.ying@sydney.edu.au\\
      \addr School of Mathematics and Statistics \\ The University of Sydney
      \AND
      \name Ding-Xuan Zhou \email dingxuan.zhou@sydney.edu.au\\
      \addr School of Mathematics and Statistics \\ The University of Sydney}

\def\month{10}  
\def\year{2026} 
\def\openreview{\url{https://openreview.net/forum?id=XXXX}} 

\begin{document}

\maketitle

\begin{abstract}
Direct Preference Optimization (DPO) learns from pairwise preferences through a logistic objective whose score-difference parameterization follows the Bradley--Terry model. We study how failure of this modeling assumption limits DPO, distinguishing structural misspecification from limited model capacity and finite-sample estimation. Our main contribution is a characterization of structural approximation error through the triangle-curl of the preference log-odds, which vanishes exactly when preferences admit a Bradley--Terry representation. We upper bound the achievable excess logistic risk by a score-class approximation term and a triangle-curl term, and establish a lower bound proportional to triangle-curl under additional boundedness assumptions. For ReLU score models, including clipped models, we derive explicit approximation rates showing that increasing capacity can eliminate the score-approximation term while leaving the structural obstruction unchanged. We further establish finite-sample bounds for clipped empirical risk minimizers that separate network approximation, structural misspecification, and estimation error. We place these results within a surrogate-risk framework by identifying contextual preference accuracy (CPA) as a natural measure of agreement between preference judgments and changes in pairwise probability ratios relative to the reference policy. Standard logistic consistency and calibration results connect DPO's objective to this policy property and transfer our excess-risk upper bounds to CPA, without treating CPA as a complete measure of generated-response quality.
\end{abstract}

\section{Introduction}
Aligning large language models (LLMs) with human
preferences commonly relies on pairwise comparisons,
where annotators indicate which of two responses
is preferred for a given prompt. Learning from
these comparisons requires a model of how preferences
relate across responses. The resulting policy is
therefore shaped not only by the available data,
but also by the restrictions imposed by this model.

Reinforcement Learning from Human Feedback
(RLHF)~\citep{christiano2017deep,stiennon2020learning,
ouyang2022training,kaufmann2023survey, lambert2025reinforcement} first fits a reward model to
preference data and then optimizes a policy using
reinforcement learning. Direct Preference Optimization
(DPO)~\citep{rafailov2023direct} simplifies this
pipeline through a single supervised objective.
Its logistic loss models preference log-odds as
differences of policy-to-reference log-density ratios,
following the Bradley--Terry (BT) model~\citep{bradley1952rank}.

The distinction between policy expressiveness and
preference representability is central.
Even an unrestricted scalar score cannot represent
arbitrary pairwise preference log-odds: score
differences must sum to zero around every triangle
of responses. Cyclic preferences violate this
condition, but even preferences with a consistent
ordering need not satisfy it.
Increasing model capacity therefore need not
eliminate misspecification.
Recent work studies misspecification arising from
restrictions on the policy class under BT
preferences~\citep{gopalan2026dpo}.
We examine the complementary question of what
happens when the preferences themselves do not
admit any score-difference representation.

Our central question is how this structural
restriction limits the logistic risk achievable
by DPO, and how that limitation can be distinguished
from insufficient model capacity and finite
preference data. We approach this question through
statistical learning theory, separating structural
misspecification, score-class approximation, and
finite-sample estimation.

To connect this risk analysis to policy behavior,
we adopt the classical surrogate-risk
framework~\citep{bartlett2006convexity,zhang2004statistical}.
We consider contextual preference accuracy (CPA),
the accuracy of the preference decisions implicit
in DPO. At the policy level, CPA measures how often
the policy increases the probability of the
preferred response relative to the other response,
compared with the reference policy.
This captures a specific aspect of policy behavior,
without guaranteeing improved generated-response
quality. Calibration connects excess logistic risk
to CPA, allowing our analysis of DPO's structural
and statistical limitations to yield guarantees
for this policy property.

\paragraph{Contributions.}
We formulate DPO within the classical surrogate-risk
framework~\citep{bartlett2006convexity,zhang2004statistical},
identifying contextual preference accuracy (CPA)
as a natural measure of agreement between preference
judgments and policy reweighting relative to the
reference. Established consistency and calibration
properties of logistic loss connect this criterion
to DPO's objective and allow our excess-risk upper
bounds to transfer to CPA. Our contribution here
is this formulation and its policy interpretation,
rather than new calibration properties of logistic loss.

Our main technical contribution is a characterization
of the structural approximation error imposed by
DPO's score-difference parameterization.
We introduce a triangle-curl functional of the
preference log-odds that vanishes exactly when
preferences admit a Bradley--Terry representation.
We upper bound the achievable excess logistic risk
by a score-class approximation term and a triangle-curl
term, and establish a lower bound proportional to
triangle-curl under additional boundedness assumptions.
For ReLU score models, including clipped models,
we derive explicit rates showing that the
capacity-dependent approximation term can be made
arbitrarily small.
We further establish finite-sample bounds for
clipped empirical risk minimizers, separating
network approximation, structural misspecification,
and estimation error.

\subsection{Related Works}
A growing literature develops alternatives to DPO,
including methods for noisy preference
data~\citep{chowdhury2024provably}, alternative
preference objectives~\citep{azar2024general,zhao2023slic},
reference-free optimization~\citep{meng2024simpo},
and joint supervised fine-tuning and preference
optimization~\citep{hong2024orpo}.
Alongside these algorithmic developments,
theoretical work has established guarantees
under specific modeling assumptions.
However, our analysis addresses the less studied question of
how failure of the Bradley--Terry assumption limits DPO.

Our analysis follows the classical surrogate-risk framework in statistical learning theory (SLT)~\citep{bartlett2006convexity,zhang2004statistical}, which studies the relationship between a tractable surrogate loss and a target population risk. Fisher consistency and calibration are well understood for binary classification and ranking, including AUC optimization~\citep{clemenccon2008ranking,gao2015consistency,YingAUC2016}, but analogous analyses of preference optimization objectives for LLM alignment remain limited.

Recent work has also examined the limitations of the Bradley--Terry (BT) model. \citet{pukdee2026does} characterize when pairwise comparison data admits a BT representation and analyze the behavior of BT learning under model misspecification. From a social choice perspective, \citet{dong2026representation} show that excess risk decomposes into representation and aggregation errors, the latter arising from Condorcet cycles in pooled preferences. Meanwhile, \citet{golz2026distortion} quantify the loss incurred when aggregating heterogeneous BT models into a single policy or ranking. These works highlight fundamental limitations of representing heterogeneous preferences with a single score-based model.

Empirically, \citet{chen2024preference} show that preference learning methods such as DPO often produce only modest improvements in preference accuracy, suggesting a persistent gap between optimization objectives and alignment performance that may partly arise from the influence of the reference model. \citet{gopalan2026dpo} study DPO under
Bradley--Terry preferences, showing how restrictions
on the policy class can prevent its induced rewards
from representing the preference-generating reward.
Our analysis addresses a distinct source of
misspecification: the preference log-odds may not
admit a score-difference representation at all.
We quantify this obstruction through triangle-curl
and separate it from score-class approximation error.
The two perspectives are complementary, distinguishing
limitations of the chosen policy class from
limitations of the score-difference representation
itself.

\section{Background and Problem Setup}
\label{sec:learning setup}

Preference-based alignment methods, including RLHF and
DPO, use judgments comparing two candidate responses
to the same prompt. We study the preference-prediction
problem implicit in the DPO objective and the
statistical consequences of its policy parameterization.

Let $\mathcal X$ denote the prompt space and
$\mathcal Y$ the response space. We observe preference
data of the form
\[
(X,Y,Y',A)\sim D,
\qquad A\in\{-1,+1\},
\]
where $A=1$ indicates that $Y$ is preferred to $Y'$,
and $A=-1$ indicates the opposite preference.
We assume that, conditional on $X$, the responses
$Y$ and $Y'$ are independently sampled from a common
distribution $p(\cdot\mid X)$.
Define the conditional preference probability
\begin{equation}
\label{eq:eta}
\eta(x,y,y')
:=
\mathbb P(A=1\mid X=x,Y=y,Y'=y').
\end{equation}

A policy assigns a conditional probability
distribution $\pi(\cdot\mid x)$ over responses to
each prompt $x$. We write $\pi(y\mid x)$ for its
probability mass or density, as appropriate.
Let $\pi_{\mathrm{ref}}$ be a fixed reference policy.
Where policy-to-reference ratios are used, we
consider policies and a reference policy with
positive masses or densities on the support of
the comparison distribution.

\paragraph{DPO as preference prediction.}
DPO uses the policy-induced score
\[
s_\pi(x,y)
:=
\tau\log
\frac{\pi(y\mid x)}{\pi_{\mathrm{ref}}(y\mid x)},
\qquad \tau>0,
\]
and the corresponding comparison function
\[
f_\pi(x,y,y')
:=
s_\pi(x,y)-s_\pi(x,y').
\]
Its population objective is
\begin{equation}
\label{eq: DPO objective}
-\mathbb E_{(X,Y,Y',A)\sim D}
\left[
\log\sigma\left(Af_\pi(X,Y,Y')\right)
\right],
\end{equation}
where $\sigma(t)=(1+e^{-t})^{-1}$
\citep{rafailov2023direct}.

Thus, DPO has the form of binary logistic
prediction with input $(X,Y,Y')$ and label $A$.
The quantity $\sigma(f_\pi(x,y,y'))$ is its
predicted probability that $y$ is preferred to
$y'$. Since $\sigma(f_\pi(x,y,y'))>1/2$ if and only if
$f_\pi(x,y,y')>0$, the sign of $f_\pi$ determines
which response the model predicts to be preferred,
with zero corresponding to a tie.
This decision compares policy-to-reference
density ratios, rather than raw policy
probabilities $\pi(y\mid x)$.

To analyse this prediction problem, we extend
the logistic objective to arbitrary measurable
triplet-wise predictors
$f:\mathcal X\times\mathcal Y\times\mathcal Y
\to\mathbb R$.
For $\ell(u)=\log(1+e^{-u})$, define
\[
\begin{aligned}
\mathcal L(f)
&:=
\mathbb E_{(X,Y,Y',A)\sim D}
\left[\ell\left(Af(X,Y,Y')\right)\right] =
-\mathbb E_{(X,Y,Y',A)\sim D}
\left[\log\sigma\left(Af(X,Y,Y')\right)\right].
\end{aligned}
\]
Substituting $f=f_\pi$ recovers
\eqref{eq: DPO objective}. DPO therefore minimizes
$\mathcal L(f_\pi)$ over the chosen policy class.

\paragraph{Contextual preference accuracy.}
The sign of $f_\pi$ determines the preference decision
implicit in DPO's logistic objective. To study these
decisions, for any measurable triplet-wise predictor
$f:\mathcal X\times\mathcal Y\times\mathcal Y\to\mathbb R$,
define its \emph{contextual preference accuracy (CPA)} as
\[
\begin{aligned}
P(f)
:=
\mathbb E_{X,Y,Y'}\Big[
&\eta(X,Y,Y')\,\mathbf{1}\{f(X,Y,Y')>0\}\\
&+\bigl(1-\eta(X,Y,Y')\bigr)\,
  \mathbf{1}\{f(X,Y,Y')<0\}\\
&+\tfrac12\,\mathbf{1}\{f(X,Y,Y')=0\}
\Big].
\end{aligned}
\]
This is the population accuracy of the induced
preference decisions, with half credit for ties.
Under DPO's parameterization, these decisions describe
how the policy reweights responses relative to the
reference: for positive policy and reference probabilities,
\[
f_\pi(x,y,y')>0
\quad\Longleftrightarrow\quad
\frac{\pi(y\mid x)}{\pi(y'\mid x)}
>
\frac{\pi_{\mathrm{ref}}(y\mid x)}
     {\pi_{\mathrm{ref}}(y'\mid x)}.
\]
Thus, $P(f_\pi)$ measures how often the policy increases
the probability of the preferred response relative to
the other response, compared with the reference policy,
with half credit for unchanged ratios.
In particular, if $P(f_\pi)=P^*:=\sup_f P(f)$, where
the supremum is over all measurable triplet-wise
predictors, then
Lemma~\ref{lem:character of pairwise preference accuracy maximizers}
implies that the displayed inequality holds for almost
every pair with $\eta(x,y,y')>1/2$.
However, as discussed by \citet{chen2024preference},
this density-ratio ordering does not imply
$\pi(y\mid x)>\pi(y'\mid x)$.
Moreover, both responses may lose probability even
when their ratio changes in the preferred direction.
CPA alone therefore does not establish an improvement
in the quality of generated responses.

The key question is how closely DPO's score-difference
predictors can approximate the optimal preference
log-odds. We study this through excess logistic risk,
using the surrogate-risk framework to connect risk
bounds to CPA. This separates the representational
limitations of DPO from the calibration properties
of the logistic loss.

\section{Fisher Consistency and Calibration}\label{sec:Calibration}
The consistency and calibration results in this
section are standard consequences of logistic
surrogate-risk theory~\citep{bartlett2006convexity,
zhang2004statistical}; we do not claim these results
as novel. Our contribution is the formulation of
DPO as a surrogate for CPA and its connection to
the policy interpretation developed above.
Working with unrestricted triplet-wise predictors
separates these properties of the loss from the
score-difference approximation issues studied in
Section~\ref{sec:score-diff approx}. First, we show Fisher consistency: minimizing the logistic risk recovers a predictor that maximizes CPA. Second, we strengthen this connection by establishing a quantitative calibration inequality, showing that excess logistic risk controls excess CPA. 
We include proofs for completeness.


\subsection{Fisher Consistency}\label{subsec: Fisher consistency}
To establish the Fisher consistency of logistic risk with respect to CPA, we characterize the population minimizers of logistic risk and the population maximizers of CPA.
The following lemma characterizes the population minimizers of logistic population risk. It is the direct three-variable (triplet-wise) analog of the characterization of functions minimizing the population logistic risk in binary classification \citep{bartlett2006convexity,zhang2004statistical}. Here, the
optimal object is the preference log-odds. 

All proofs of the results that follow can be found in the Appendix.


\begin{lemma}
\label{lem:character of pairwise logistic risk minimizers}
A measurable triplet-wise function \(f\) minimizes the population logistic risk if and only if
\[
f(x,y,y') = \log \frac{\eta(x,y,y')}{1 - \eta(x,y,y')}
\quad \text{a.s.},
\]
with the convention that the right hand side is $+\infty$ when $\eta(x,y,y') = 1$ and $-\infty$ when $\eta(x,y,y') = 0$. 
\end{lemma}

Define $f^*$ to be the logistic risk-minimizing preference log-odds function 
\[
f^*(x,y,y') = \log\frac{\eta(x,y,y')}{1-\eta(x,y,y')}.
\]

\begin{lemma}
\label{lem:character of pairwise preference accuracy maximizers}
A measurable triplet-wise function $f$ maximizes $P(f)$
if and only if
\[
\mathrm{sign}(f(x,y,y')) = \mathrm{sign}(2\eta(x,y,y') - 1) \quad \text{a.s.}
\]
on the set where
$ 
\eta(x,y,y')\neq \frac12,$
with arbitrary tie-breaking allowed on the set where
$\eta(x,y,y')=1/2$. Moreover,
\[
P^* := \sup_f P(f)
=
\mathbb E\left[
\max\{\eta(X,Y,Y'),1-\eta(X,Y,Y')\}
\right].
\]
\end{lemma}

The optimal value $P^*$ is the pairwise analog of the optimal Bayes accuracy in binary classification and represents the highest achievable agreement with human preferences.

Combining Lemma~\ref{lem:character of pairwise logistic risk minimizers} with Lemma~\ref{lem:character of pairwise preference accuracy maximizers}, we can establish Fisher consistency in the following theorem. Specifically, minimizing the logistic risk recovers the true preference log-odds \(f^*(x,y,y')=\log\frac{\eta(x,y,y')}{1-\eta(x,y,y')}\), whose sign coincides almost surely with \(\sgn(2\eta(x,y,y')-1)\), implying that every logistic risk minimizer maximizes the CPA.

\begin{theorem}
\label{thm:dpo-fisher-delta}
Every minimizer of the population logistic risk maximizes the
population CPA:
\[
\arg\min_f \mathcal L(f)
\subseteq
\arg\max_fP(f),
\]
where we consider measurable triplet-wise functions $f:\mathcal X\times\mathcal Y\times\mathcal Y\to\mathbb R.$

\end{theorem}

\subsection{Calibration Inequality}\label{subsec: calibration ineq}

The Fisher consistency result of Theorem~\ref{thm:dpo-fisher-delta} only demonstrates consistency in a qualitative way; it does not quantify how optimization error in the surrogate risk transfers to preference alignment performance.
We now strengthen this connection by establishing a calibration inequality, which controls excess CPA in terms of the excess logistic risk. 

Define
\[
\mathcal L^* := \inf_f \mathcal L(f),
\]
where the infimum is over all measurable
triplet-wise predictors $f:\mathcal X\times
\mathcal Y\times\mathcal Y\to\mathbb R$.

\begin{theorem}
\label{thm:dpo-calibration-delta}
For any measurable triplet-wise function $f$, there holds
\[
P^*-P(f)
\le
\sqrt{2}\,
\sqrt{\mathcal L(f)-\mathcal L^*}.
\]

\end{theorem}

\begin{remark}
The calibration inequality implies statistical consistency. For any
sequence of measurable triplet-wise functions $\{f_n\}_{n\ge1}$,
\[
\mathcal L(f_n)\to\mathcal L^*
\quad\Longrightarrow\quad
P(f_n)\to  P^*.
\]
Thus, convergence of the logistic risk guarantees convergence to
the optimal CPA.
\end{remark}


\section{Approximation Error of Score-Difference DPO}\label{sec:score-diff approx}

For any score function $s:\mathcal X\times\mathcal Y\to\mathbb R$,
define
\[
\Delta_s(x,y,y'):=s(x,y)-s(x,y').
\] Section~\ref{sec:Calibration} proves consistency and calibration for unrestricted triplet-wise functions, however,
DPO restricts the predictor to $f_\pi=\Delta_{s_\pi}$ for the policy induced score functions $s_\pi$.
By Theorem~\ref{thm:dpo-calibration-delta},
\[
P^*-\sup_s P(\Delta_s)
\le
\sqrt{2}\,
\Bigl(
\inf_s \mathcal L(\Delta_s)-\mathcal L^*
\Bigr)^{1/2}.
\]
Thus, even when restricting triplet-wise functions to the score-differences, excess CPA is controlled by the surrogate-risk approximation error. From Lemma~\ref{lem:character of pairwise logistic risk minimizers}, $f^*$ is the preference log-odds and global minimizer of $\L(f)$, so to understand the surrogate risk approximation error, we need to understand the limitations of approximating $f^*$ with a score-difference function.

The key message of this section is that restricting to score-difference functions imposes an intrinsic structural limitation on the triplet-wise preference functions that can be represented, regardless of model capacity. A score-difference function assigns each response a scalar latent score and compares responses through score differences. Consequently, score differences satisfy a structural consistency condition: for any three responses,
\[
f(x,y_1,y_2)+f(x,y_2,y_3)+f(x,y_3,y_1)=0.
\]
Thus, not every triplet-wise preference function admits a score-difference representation, and the ability of a triplet-wise function to admit a score-difference representation is governed by its behavior over all response triples.
Naturally, we ask to what extent does $f^*$ fail this condition and how this affects the performance of DPO.



We formalize this intuition by introducing triangle-curl
as a measure of score-model misspecification, using
combinatorial curl from combinatorial Hodge
theory~\citep{jiang2011statistical}.
Theorem~\ref{thm:curl_characterization} characterizes
score-difference functions through vanishing triangle-curl;
in particular, $\mathcal I(f^*)=0$ if and only if
preferences admit a Bradley--Terry representation.
Theorem~\ref{thm:triangle-curl bounds} relates this
quantity to surrogate-risk approximation error,
providing an upper bound and, under additional
boundedness assumptions, a lower bound.
We then specialize these bounds to ReLU score
models and their induced DPO policies in
Theorem~\ref{thm:score_differ}, establishing
approximation upper bounds for both unclipped
and clipped scores, together with a triangle-curl
lower bound for the clipped class under additional
boundedness assumptions on the preferences.
Finally, Theorem~\ref{thm:generalization} combines
the clipped-class approximation bound with
finite-sample estimation bounds and transfers
the resulting excess-risk guarantee to CPA
through calibration.

\paragraph{Triangle-Curl.}
For a triplet-wise function $f(x,y,y')$, we apply the combinatorial curl operator from combinatorial Hodge theory~\citep{jiang2011statistical} to $f(x,\cdot,\cdot)$:
\[
(\mathrm{curl}\, f)(x;y_1,y_2,y_3)
:=
f(x,y_1,y_2) + f(x,y_2,y_3) + f(x,y_3,y_1).
\label{eq:pointwise-curl}
\]
The combinatorial curl characterizes when a triplet-wise function can be represented by a scalar score function. Intuitively, score differences always telescope around a triangle of responses, so any score-difference model is curl-free. The following lemma shows that the converse also holds: vanishing curl is exactly the condition under which a triplet-wise function admits a score-difference representation.

\begin{lemma}\label{lem: comb curl char}
Let $f:\X \times \Y\times \Y\to\mathbb{R}$ be antisymmetric in its response arguments. Then $\mathrm{curl}\,f=0$ if and only if there exists a score function $s:\X\times \Y\to\mathbb{R}$ such that
    \[
    f(x,y,y')=s(x,y)-s(x,y')
    \]
    for all $y,y'\in \Y$.
\end{lemma}

For conditionally i.i.d.\ $Y_1,Y_2,Y_3\sim p(\cdot\mid X)$, define the \emph{triangle-curl} of a square-integrable triplet-wise function $f$ as
\[
\mathcal I(f)
:=
\E_{X,Y_1,Y_2,Y_3}
\left[
(\mathrm{curl}\,f)(X;Y_1,Y_2,Y_3)^2
\right].
\]
Because combinatorial curl can be negative, this quantity aggregates the squared combinatorial curl of $f$ across random response triples to measure the overall inability of $f$ to admit a score-difference representation. The key quantity that we are interested in is $\mathcal{I}(f^*)$. Since $f^*$ is the preference log-odds, not admitting a score-difference representation is an underlying property of the preferences. In particular,
$
\eta(x,y,y') = \sigma(f^*(x,y,y')),
$
and $\mathcal{I}(f^*)$ can be interpreted as measuring the extent to which the preferences do not follow the Bradley--Terry model~\citep{bradley1952rank}:
\[
\eta(x,y,y') = \sigma(s(x,y) - s(x,y')).
\]

We have a global analogy to Lemma~\ref{lem: comb curl char} using triangle-curl.
\begin{theorem}
\label{thm:curl_characterization}
Let $f:\X \times \Y \times \Y \rightarrow \R$ be square-integrable and satisfy $f(x,y,y') = -f(x,y'y)$ almost surely. Then 
$\mathcal I(f)=0$ if and only if there exists a measurable function
$
s:\mathcal X\times\mathcal Y\to\mathbb R
$
such that
\[
f(x,y,y')
=
s(x,y)-s(x,y')
\qquad\text{a.s.}
\]
In particular, under Assumption~\ref{ass:L_2}, $\mathcal{I}(f^*) = 0$ if and only if preferences almost surely follow the Bradley--Terry model:
there exists a measurable function
$
s:\mathcal X\times\mathcal Y\to\mathbb R
$ 
such that
\[
\eta(x,y,y') = \sigma(s(x,y) - s(x,y')) \quad \mathrm{a.s.}
\]
\end{theorem}
Theorem~\ref{thm:curl_characterization} characterizes score-difference
representations independently of the learning
objective. In particular, let $g^*$ be a
square-integrable antisymmetric population
minimizer of an unrestricted preference-learning
objective. Then $\mathcal I(g^*)=0$ if and only
if $g^*$ admits an exact score-difference
representation. This provides a representability
criterion for other preference-learning methods
that use score differences, although its
implications for excess risk depend on the
chosen loss.

\begin{assumption}\label{ass:L_2}
The preference log-odds function $f^*$ is square-integrable, and hence finite almost surely,
so that $\eta(x,y,y')\in(0,1)$ almost surely. All score functions considered are also assumed
to be square-integrable. Finally, preferences are reciprocal, i.e.
$\eta(x,y',y)=1-\eta(x,y,y')$, which implies that $f^*$ is antisymmetric:
$f^*(x,y,y')=-f^*(x,y',y)$ for all $x,y,y'$.
\end{assumption}
Under Assumption~\ref{ass:L_2}, define the centered score potential
\begin{align}\label{eq:s_star_def}
    s^*(x,y)
    :=
    \E_{U\sim p(\cdot\mid x)}
    \bigl[
        f^*(x,y,U)
    \bigr].
\end{align}
The induced score difference $\Delta_{s^*}$ is
the $L^2$ projection of $f^*$ onto the
score-difference subspace, while
$\mathcal I(f^*)/3$ is the squared $L^2$ norm
of the orthogonal residual; see Lemmas~\ref{lem:proj onto score diff} and~\ref{lem:struc_cycle} in Appendix~\ref{proof: thm triangle-curl bounds}. In this sense, $\mathcal{I}(f^*)$ measures the intrinsic component of $f^*$ that cannot be represented as a score difference.

\begin{theorem}\label{thm:triangle-curl bounds}
For any class $\mathcal S$ of square-integrable score functions, the population excess logistic risk satisfies
\[
\inf_{s\in\mathcal S}\bigl(\L(\Delta_s)-\L^*\bigr)
\;\le\;\frac{1}{8}
\inf_{s\in\mathcal S}\mathbb{E}\!\left[(\Delta_s-\Delta_{s^*})^2\right]
+\frac{1}{24}\mathcal I(f^*).
\]
Moreover, if for every $s\in\mathcal S$,
$\delta<\sigma(\Delta_s(X,Y,Y'))<1-\delta$ and
$\delta<\sigma(f^*(X,Y,Y'))<1-\delta$  for some $\delta\in(0,1/2)$ almost surely, then
\[
\frac{\delta^2}{6}\mathcal I(f^*)
\;\le\;
\inf_{s\in\mathcal S}\bigl(\L(\Delta_s)-\L^*\bigr).
\]
\end{theorem}
Theorem~\ref{thm:triangle-curl bounds} reduces the score-difference approximation error into two terms: an approximation error measuring how well the score-difference component $\Delta_{s^*}$ of $f^*$ can be approximated by the score class $S$, and an intrinsic error given by the triangle-curl. The first term decreases with model expressiveness, whereas the triangle-curl depends only on the underlying preference distribution. Under the additional boundedness conditions,
nonzero triangle-curl yields a positive lower
bound on the achievable excess logistic risk.
The structural obstruction vanishes exactly
when preferences admit a Bradley--Terry
representation, although approximation error
may remain within a restricted score class.

We now explicitly bound the surrogate risk approximation error, taking $\mathcal{S}$ to be a class of deep ReLU neural networks.

\paragraph{Hypothesis space.}
Let $\Omega:=\X\times\Y\subset \R^d$, $z=(x,y)$, and let $\phi(t)=\max\{0,t\}$. For any $m\in\N$ and $\v\in\R^m$, define
\[
    \phi_{\v}:\R^m\to [0,\infty)^m,
    \qquad
    z\mapsto
    \bigl(\phi((z-\v)_1),\ldots,\phi((z-\v)_m)\bigr).
\]
A ReLU neural network is a function $s:\R^d\to\R$ of the form
\begin{align}\label{eq:relu_network}
    s(z)
    =
    \W_L\phi_{\v_L}\W_{L-1}\phi_{\v_{L-1}}
    \cdots
    \W_1\phi_{\v_1}\W_0 z,
\end{align}
where $L\in\N$, $m_0=d$, $m_{L+1}=1$, $\v_i\in\R^{m_i}$, and
$\W_i\in\R^{m_{i+1}\times m_i}$ for $0\le i\le L$. For $L,N\in\N$, define
\begin{align}\label{eq:S_d_L_N}
    \mathcal S(d,L,N)
    :=
    \bigl\{
        s:\R^d\to\R
        \mid
        s \text{ is of the form \eqref{eq:relu_network} and } m_i\le N
        \text{ for }1\le i\le L
    \bigr\}.
\end{align}
The corresponding triplet-wise class is
\[
    \Delta\mathcal S(d,L,N)
    :=
    \bigl\{
        \Delta_s
        \mid
        s\in\mathcal S(d,L,N)
    \bigr\}.
\]
We impose some regularity assumption on the optimal scoring function $s^*$. In this work, we consider H\"older continuity assumption, which is standard in the literature \citep{nagler2026optimal, schmidt2020nonparametric,lu2021deep}.
For $\beta\in(0,1]$ and $h:\Omega \rightarrow \R$, the H\"older norm of $h$ on $\Omega$ is defined as
\[
    \|h\|_{\beta,\Omega}
    :=
    \max\left\{
        \sup_{z\in\Omega}|h(z)|,
        \sup_{z_1\neq z_2\in\Omega}
        \frac{|h(z_1)-h(z_2)|}{\|z_1-z_2\|_\infty^\beta}
    \right\}.
\]
\begin{assumption}\label{ass:holder}
There exist $R>0$, $\beta\in(0,1]$, and $B>0$ such that $\Omega=\X\times\Y\subseteq[-R,R]^d$ and $\|s^*\|_{\beta,\Omega}\le B$.
\end{assumption}

Let $\Theta(d,L,N)$ denote the parameter space corresponding to ReLU networks $s\in \mathcal{S}(d,L,N)$. Each parameter vector $\theta \in \Theta(d,L,N)$ induces a score function $s_\theta$, which in turn induces a policy
\[
\pi_\theta(y \mid x)
=
\frac{\pi_{\mathrm{ref}}(y \mid x)\exp(s_\theta(x,y)/\tau)}{Z_\theta(x)},
\qquad
Z_\theta(x) = \int_Y \pi_{\mathrm{ref}}(u \mid x)\exp(s_\theta(x,u)/\tau)\,du.
\]
We then have the triplet-wise score-difference function
$
f_{\pi_\theta}(x,y,y') = s_\theta(x,y) - s_\theta(x,y').
$

We cannot invoke the lower bound of Theorem~\ref{thm:triangle-curl bounds} for the score class $\mathcal S$, as there is no uniform bound of the score values, and thus $\sigma(\Delta_s(x,y,y'))$ can be made arbitrarily close to zero and one. Therefore, we additionally consider bounded ReLU network score functions. For $\theta\in\Theta(d,L,N)$, define
\[
s_{\theta,B}
:=
\operatorname{clip}_B\circ s_\theta,
\qquad
\operatorname{clip}_B(t)
:=
\max\{-B,\min\{B,t\}\},
\]
so that $\sup_{(x,y) \in \mathcal X \times \mathcal Y}|s_{\theta,B}(x,y)| \leq B$ and $B$ is the same constant from Assumption~\ref{ass:holder}. Let $\pi_{\theta,B}$ denote the policy induced by
$s_{\theta,B}$ through the same exponential
parameterization, so that
$f_{\pi_{\theta,B}}=\Delta_{s_{\theta,B}}$.

\begin{theorem}\label{thm:score_differ}
Let Assumptions~\ref{ass:L_2} and~\ref{ass:holder} hold. Then there exist constants $C_1, C_2 > 0$, depending only on $d, \beta, B, R$, such that for every $L_0, N_0 \ge 2$, if
\[
\bar L = C_1 L_0 + C_1 L_0 \left\lceil \frac{\sqrt{\log L_0}}{\sqrt{\log N_0}} \right\rceil,
\qquad
\bar N = C_1 N_0,
\]
then
\[
\inf_{\theta \in \Theta(d,\bar L,\bar N)}
\bigl[ \L(f_{\pi_\theta}) - \L^\ast \bigr]
\;\le\;
C_2 \left(\frac{1}{L_0^2 N_0^2}\right)^{2\beta/d}
\;+\;
\frac{1}{24} \mathcal I(f^\ast).
\]
Consequently, for every $\varepsilon \in (0,1)$, there exist network sizes $\bar L_\varepsilon, \bar N_\varepsilon$ such that
\[
\inf_{\theta \in \Theta(d,\bar L_\varepsilon,\bar N_\varepsilon)}
\bigl[ \L(f_{\pi_\theta}) - \L^\ast \bigr]
\;\lesssim\;
\varepsilon + \mathcal I(f^\ast).
\]
The above two upper bounds also hold true when $f_{\pi_\theta}$ is replaced by $f_{\pi_{\theta,B}}$ for the clipped score class. If additionally $a\le\eta(X,Y,Y')\le1-a$ almost surely
for some $a\in(0,1/2)$, then, for any
$0<\delta<\min\{a,\sigma(-2B)\}$, the excess DPO risk of the clipped score class is bounded below:
\[
\frac{\delta^2}{6}\mathcal I(f^*)
\le
\inf_{\theta\in\Theta(d,\bar L,\bar N)}
\bigl[\mathcal L(f_{\pi_{\theta,B}})-\mathcal L^*\bigr].
\]
\end{theorem}

Theorem~\ref{thm:score_differ} separates a capacity-dependent
approximation term, which vanishes as network size increases,
from a structural term determined by triangle-curl.
The upper bounds hold for both the original and clipped
parameterizations. For the clipped class, if the preference
probabilities are bounded away from zero and one, the
matching lower bound shows that nonzero triangle-curl
creates an irreducible excess-risk floor, regardless of
network capacity. Under these additional conditions, triangle-curl therefore
controls the best achievable excess DPO risk up to
multiplicative constants as the capacity-dependent term
vanishes.

\subsection{Finite-Sample Generalization}
\label{sec:generalization}

We now combine the approximation bounds with a
finite-sample estimation bound for empirical risk
minimization. We use the clipped parameterization
introduced above to ensure that the logistic loss
is uniformly bounded.

Fix $L_0,N_0\ge2$, with corresponding network
dimensions $\bar L,\bar N$ from
Theorem~\ref{thm:score_differ}, and write
\[
\mathcal S_B
:=
\{s_{\theta,B}:\theta\in\Theta(d,\bar L,\bar N)\},
\qquad
\Delta\mathcal S_B:=\{\Delta_s:s\in\mathcal S_B\}.
\]
Let $D_n=\{(X_i,Y_i,Y_i',A_i)\}_{i=1}^n$ be
i.i.d. samples from $D$, under the conditional
sampling model specified in Section~2.
Define the empirical logistic risk
\[
\widehat{\mathcal L}_n(f)
:=
\frac1n\sum_{i=1}^n
\log\left(1+\exp\{-A_i f(X_i,Y_i,Y_i')\}\right).
\]
Suppose a measurable empirical risk minimizer exists,
and let
\[
\widehat\theta
\in
\operatorname*{arg\,min}_{\theta\in\Theta(d,\bar L,\bar N)}
\widehat{\mathcal L}_n(f_{\pi_{\theta,B}}).
\]

Write
\[
\mathfrak R_n(\mathcal S_B)
:=
\mathbb E\left[
\sup_{s\in\mathcal S_B}
\frac1n\sum_{i=1}^n\epsilon_i s(X_i,Y_i)
\right]
\]
for the expected Rademacher complexity, where
$\epsilon_1,\ldots,\epsilon_n$ are independent uniform
signs, independent of the observations. The expectation
is over both the observations and the signs.

\begin{theorem}[Finite-sample generalization]
\label{thm:generalization}
Let Assumptions~\ref{ass:L_2} and~\ref{ass:holder} hold. Set $M=\log(1+e^{2B})$. For any $\alpha\in(0,1)$,
with probability at least $1-\alpha$,
\[
\begin{aligned}
\mathcal L(f_{\pi_{\widehat\theta,B}})
-\mathcal L^*
\le{}&
C_2\left(\frac{1}{L_0^2N_0^2}\right)^{2\beta/d}
+\frac1{24}\mathcal I(f^*)\\
&+
8\mathfrak R_n(\mathcal S_B)
+
2M\sqrt{\frac{\log(2/\alpha)}{2n}},
\end{aligned}
\]
where $C_2$ is the constant in
Theorem~\ref{thm:score_differ}.
On the same event,
\[
\begin{aligned}
P^*-P(f_{\pi_{\widehat\theta,B}})
\le \sqrt{2}\Bigg[
&
C_2\left(\frac{1}{L_0^2N_0^2}\right)^{2\beta/d}
+\frac1{24}\mathcal I(f^*)\\
&+
8\mathfrak R_n(\mathcal S_B)
+
2M\sqrt{\frac{\log(2/\alpha)}{2n}}
\Bigg]^{1/2}.
\end{aligned}
\]
\end{theorem}

Theorem~\ref{thm:generalization} separates the
capacity-dependent approximation term, the structural
triangle-curl term, and finite-sample estimation error.
Increasing network capacity reduces the first term
but can increase the Rademacher complexity.
The theorem leaves this trade-off explicit rather
than prescribing depth and width as functions of $n$.
Unlike the lower bound in
Theorem~\ref{thm:score_differ}, these upper bounds
do not require preference probabilities to be
bounded away from zero and one.

For an approximate empirical minimizer with optimization error $\varepsilon_{\mathrm{opt}}$, we have
\[
\widehat{\mathcal L}_n(f_{\pi_{\widehat\theta,B}})
\le
\inf_{\theta\in\Theta(d,\bar L,\bar N)}
\widehat{\mathcal L}_n(f_{\pi_{\theta,B}})
+\varepsilon_{\mathrm{opt}},
\]
and the same conclusions of Theorem~\ref{thm:generalization} hold with
$\varepsilon_{\mathrm{opt}}$ added to the
excess-risk bound and inside the square brackets
of the CPA bound.

\section{Conclusion}
We studied Direct Preference Optimization through the lens of statistical learning theory. Using standard consistency and calibration properties of logistic risk, we connected DPO's objective to contextual preference accuracy, which measures agreement between preference judgments and policy reweighting relative to the reference. For score-difference parameterizations, we characterized approximation error through a capacity-dependent score-approximation term and an intrinsic triangle-curl term, which vanishes exactly under the Bradley--Terry model. For ReLU score models, the capacity-dependent term can be made arbitrarily small, while triangle-curl yields an irreducible approximation barrier under additional boundedness assumptions. We also established finite-sample excess-risk bounds for clipped empirical risk minimizers, separating approximation, structural misspecification, and estimation error, and transferred these bounds to CPA through calibration.

Future work includes selecting network capacity as a function of sample size to obtain explicit convergence rates, developing richer preference models beyond score differences that can represent nonzero triangle-curl, and constructing estimators of triangle-curl as a practical diagnostic for violations of the Bradley--Terry assumption. Finally, Assumption~\ref{ass:L_2} excludes deterministic preferences by requiring square-integrable preference log-odds, and hence $\eta\in(0,1)$ almost surely. An interesting direction is to relax this assumption, for example by adapting truncation techniques similar to those of~\citet{zhang2024classification}.

\bibliography{main}
\bibliographystyle{tmlr}

\appendix
\appendix

\section{Omitted Proofs}\label{apd:first}
For the fixed triple $(x,y,y')$, define the \emph{conditional logistic risk}
$
\varphi_\eta(u)
:=
\eta \,\ell(u)
+
(1-\eta)\,\ell(-u),
$
where $\eta = \eta(x,y,y')$ and $u\in \mathbb{R}$. We have
$
\mathcal L(f)
=
\mathbb{E}_{X,Y,Y'}
\!\left[
\varphi_{\eta(X,Y,Y')}
\!\bigl(
f(X,Y,Y')
\bigr)
\right].
$
The conditional logistic risk provides a convenient pointwise characterization of the population risk. We similarly define the \emph{conditional CPA} 
$
P_\eta(u)=\eta\,\mathbf 1\{u>0\}+(1-\eta)\,\mathbf 1\{u<0\}+\frac12\,\mathbf 1\{u=0\},
$
and likewise
$
    P(f) = \E_{X,Y,Y'}[P_{\eta(X,Y,Y')}(f(X,Y,Y'))].
$
Let $P^*=\sup_f P(f)$ and $P_\eta^*=\sup_{v\in \mathbb R} P_\eta(v)$,
where the first supremum is taken over all measurable triplet-wise functions.

\subsection{Proof of Lemma~\ref{lem:character of pairwise logistic risk minimizers}}

\begin{proof}
Let $\eta=\eta(x,y,y')$ and $u=f(x,y,y')$.
Direct calculations show that the derivative of the conditional logistic risk is $\varphi_\eta'(u)=[-\eta/(1+\exp(u))+(1-\eta)/(1+\exp(-u))]$.
Setting $\varphi_\eta'(u)=0$ gives $\eta/(1+\exp(u))=(1-\eta)/(1+\exp(-u))$.
Let
$
a=\exp(u)>0.
$
Then $\eta/(1+a)=(1-\eta)/(1+a^{-1})=(1-\eta)a/(1+a)$.
Therefore
$
a={\eta}/({1-\eta}),
$
and the stationary point $u^*$ is given by
$
u^*
=
\log \left({\eta}/({1-\eta})\right).
$ It remains to verify that $u^*$ is indeed the minimizer. Differentiating again gives $\varphi_\eta''(u)=[\eta\exp(u)/(1+\exp(u))^2+(1-\eta)\exp(-u)/(1+\exp(-u))^2\}$.
For $\eta\in(0,1)$, this quantity is strictly positive for all
finite $u$. Hence $\varphi_\eta$ is strictly convex, and the stationary
point is the unique minimizer.

If $\eta=1$, then $\varphi_\eta(u)=\ell(u)$, which is minimized in the limiting sense by sending $u\to+\infty$.
If $\eta=0$, then $\varphi_\eta(u)=\ell(-u)$, which is minimized in the limiting sense by sending $u\to-\infty$. 
Taking the expectation over the pointwise minimizer $u^*$ gives the desired result.
\end{proof}

\subsection{Proof of Lemma~\ref{lem:character of pairwise preference accuracy maximizers}}\label{lem proof: char PPA}
\begin{proof}
Let
$
\eta=\eta(x,y,y')
$
and
$
u=f(x,y,y').$ 
If $u>0$, then conditional CPA is $\eta$. If $u<0$, then the conditional CPA is $1-\eta$. If $u=0$, the conditional CPA takes value $1/2$. Therefore, the optimal conditional decision is $u>0$ if $\eta>1/2$, $u<0$ if $\eta<1/2$, and any tie-breaking decision is optimal if $\eta=1/2$. The optimal
conditional value is
$
P_\eta^* = \max\{\eta,1-\eta\}.$ 
Equivalently, on the set $\eta\neq 1/2$, optimality is characterized
by
$\operatorname{sign}(u)
=
\operatorname{sign}\!\left(\eta-\frac12\right).$ 
The lemma follows from the conditional CPA characterization of CPA.
\end{proof}

\subsection{Proof of Theorem~\ref{thm:dpo-fisher-delta}}\label{thm proof: fish}
\vspace*{-2mm}
\begin{proof}
By Lemma~\ref{lem:character of pairwise logistic risk minimizers}, any population minimizer $g^*$ of the logistic risk satisfies $g^*=\log(\eta/(1-\eta))$ almost surely, with $(x,y,y')$ suppressed. Since the logarithm is strictly increasing, $\mathrm{sign}(g^*)=\mathrm{sign}(\eta-\frac12)$ on the set $\eta\neq 1/2$. Also, if $\eta = 1$ then $g^*(x,y,y') = +\infty$ and signs agree, likewise for $\eta = 0$.
By Lemma~\ref{lem:character of pairwise preference accuracy maximizers}, this is precisely the
necessary and sufficient condition for maximizing the CPA.
\end{proof}


\subsection{Proof of Theorem~\ref{thm:dpo-calibration-delta}}\label{thm proof: cal}
\begin{lemma}\label{lem: conditional preference regret bound}
Let
$\eta = \eta(x,y,y')
$
and
$
q = \sigma(u).
$ 
Then,
$P_\eta^* - P_\eta(u)
\le
2|\eta-q|.$
\end{lemma}
\begin{proof}
Recall that
$P_\eta(u)
=
\eta \mathbf{1}\{u>0\}
+
(1-\eta)\mathbf{1}\{u<0\}
+
\frac12 \mathbf{1}\{u=0\},$
and from the proof of Lemma~\ref{lem:character of pairwise preference accuracy maximizers},
$
P_\eta^\ast=\max\{\eta,1-\eta\}.
$
Since $\sigma$ is strictly increasing, $u>0 \iff q>\frac12$ and $u<0 \iff q<\frac12$.
We consider three cases. Suppose $q>\frac12$. Then $u>0$, so $P_\eta(u)=\eta$. Therefore
$
P_\eta^\ast-P_\eta(u)
=
\max\{\eta,1-\eta\}-\eta,
$ 
which if $\eta\ge \frac12$, this is zero. Otherwise,
$
P_\eta^\ast-P_\eta(u)
=
1-2\eta
=
2\Bigl(\frac12-\eta\Bigr)
<
2(q-\eta)
=
2|\eta-q|,
$ 
since $q>\frac12$. When $q<\frac12$, the argument is analogous. 
If $q=\frac12$, then $u=0$, and therefore
$
P_\eta(u)=\frac12.
$
Thus, $P_\eta^\ast-P_\eta(u)
=
\max\{\eta,1-\eta\}-\frac12
=
\left|\eta-\frac12\right|
=
|\eta-q|
\le
2|\eta-q|.$
Combining the three cases proves the claim.
\end{proof}

Let $\eta=\eta(x,y,y')$, $u=f(x,y,y')$, and $q=\sigma(u)$. The conditional logistic risk can be written as $\varphi_\eta(u)=-[\eta \log q+(1-\eta)\log(1-q)]$.
In the proof of Lemma~\ref{lem:character of pairwise logistic risk minimizers}, we show that the conditional logistic risk is minimized when $q=\eta$. Therefore,
$\inf_{v} \varphi_\eta(v)
=
-\Bigl[\eta\log\eta+(1-\eta)\log(1-\eta)\Bigr],
$
where the infimum is over $v\in \R$. Subtracting gives $\varphi_\eta(u)-\inf_{v}\varphi_\eta(v)=\KL(\Bern(\eta)\,\|\,\Bern(q))=:K_{\eta,q}$.
By Pinsker's inequality, $|\eta-q|\le \left(K_{\eta,q}/2\right)^{1/2}$, and thus $|\eta-q|\le \left[(\varphi_\eta(u)-\inf_{v} \varphi_\eta(v))/2\right]^{1/2}$.
Combining the above inequality with Lemma~\ref{lem: conditional preference regret bound}, which says
$
P_\eta^* - P_\eta(u)
\le
2|\eta-q|,
$
we obtain $P_\eta^* - P_\eta(u)\le 2^{1/2}\left(\varphi_\eta(u)-\inf_v \varphi_\eta(v)\right)^{1/2}$.
Under the assumption now that $u = f(X,Y,Y')$, taking the expectation over $(X,Y,Y')$ gives
$
P^* - P(f)
=
\mathbb{E}\bigl[P_{\eta}^* - P_{\eta}(u)\bigr]
\leq
2^{1/2}
\mathbb{E}\left[
\left(\varphi_{\eta}(u) - \inf_v \varphi_{\eta}(v)\right)^{1/2}\right]
.
$
Applying Jensen's inequality to the right hand side
and using $\mathbb{E}[\varphi_{\eta}(u)] = \mathcal L(f)$ and $\mathbb{E}[\inf_v \varphi_{\eta}(v)] = \mathcal L^*$ proves Theorem~\ref{thm:dpo-calibration-delta}.


\subsection{Proof of Lemma~\ref{lem: comb curl char}}
\begin{proof}
Suppose $(\mathrm{curl}\,f)(x;y_1,y_2,y_3)=0$ for all $x,y_1,y_2,y_3$. Fix an arbitrary reference response $y_0\in Y$ and define $s(x,y):=f(x,y,y_0)$. Then
$
f(x,y,y')+f(x,y',y_0)+f(x,y_0,y)=0.
$
Using antisymmetry,
$
f(x,y_0,y)=-f(x,y,y_0)
$
so that
$
f(x,y,y')=s(x,y)-s(x,y').
$ Conversely, if $f(x,y,y')=s(x,y)-s(x,y')$, then scores telescope to give $\mathrm{curl}\,f = 0$.
\end{proof}


\subsection{Proof of Theorem~\ref{thm:curl_characterization}}
\begin{proof}
Since
$
\mathcal I(f)
=
\mathbb E\!\left[(\operatorname{curl}f)^2\right],
$
$
\mathcal I(f)=0
$ 
if and only if
$
\operatorname{curl}f=0
$
almost surely.

\emph{(1 $\Rightarrow$ 2).}
Assume $\mathcal I(f)=0$. By Fubini's theorem, for almost every $x$,
$
(\operatorname{curl}f)(x;y_1,y_2,y_3)=0
$
for $p(\cdot\mid x)^3$-almost every $(y_1,y_2,y_3)$. Define
$
s(x,y)
:=
\mathbb E_{U\sim p(\cdot\mid x)}
\bigl[f(x,y,U)\bigr].
$
Fix an $x$ such that
$
(\operatorname{curl}f)(x;y,y',u) = 0
$
for $p(\cdot\mid x)^3$-almost every $(y,y',u)$. Using antisymmetry,
$
f(x,u,y)
=
-f(x,y,u),
$
so that
$
f(x,y,y')
=
f(x,y,u)-f(x,y',u).
$
Taking expectations over $U\sim p(\cdot\mid x)$,
$
f(x,y,y')
=
\mathbb E[f(x,y,U)]
-
\mathbb E[f(x,y',U)]
=
s(x,y)-s(x,y')
$
almost surely.

\medskip

\noindent
\emph{(2 $\Rightarrow$ 1).}
Conversely, suppose there exists a measurable function $s$ such that
$
f(x,y,y')
=
s(x,y)-s(x,y')
$
almost surely. Then score differences telescope to give $\operatorname{curl}f = 0$
almost surely.
\end{proof}


\subsection{Proof of Theorem~\ref{thm:triangle-curl bounds}}\label{proof: thm triangle-curl bounds}
We begin with a few lemmas from which the result will follow easily.

\begin{lemma}\label{lem:delta-approx}
Let $\mathcal{S}$ be a class of score functions. There holds
\[
\inf_{s\in \mathcal S}
\mathcal L(\Delta_s)-\mathcal L^*\le \frac{1}{8}
\inf_{s\in \mathcal S}
\E_{X,Y,Y'}
\left[
    \bigl(\Delta_s(X,Y,Y')-f^*(X,Y,Y')\bigr)^2
\right].
\]
Moreover,  suppose that for every $s\in\mathcal S$,
$\delta<\sigma(\Delta_s(X,Y,Y'))<1-\delta$ and
$\delta<\sigma(f^*(X,Y,Y'))<1-\delta$  for some $\delta\in(0,1/2)$ almost surely.
\[
\inf_{s\in \mathcal S}
\mathcal L(\Delta_s)-\mathcal L^*
\ge
\frac{\delta^2}{2}
\inf_{s\in \mathcal S}
\E_{X,Y,Y'}
\left[
    \bigl(\Delta_s(X,Y,Y')-f^*(X,Y,Y')\bigr)^2
\right].
\]
\end{lemma}
\begin{proof}
Fix the triple $(x,y,y')$ and write $\eta=\eta(x,y,y')$. Using the conditional logistic risk,
for any score function $s$,
$
    \mathcal L(\Delta_s)
    =
    \E
    \left[
        \varphi_{\eta(X,Y,Y')}
        \bigl(\Delta_s(X,Y,Y')\bigr)
    \right].
$
Since
$
    \varphi_\eta'(u)=\sigma(u)-\eta,
$
$\varphi_\eta$ is minimized at
$
    f^*(x,y,y')
    =
    \log\frac{\eta(x,y,y')}{1-\eta(x,y,y')}.
$
Hence
$
    \mathcal L^*
    =
    \E
    \left[
        \varphi_{\eta(X,Y,Y')}
        \bigl(f^*(X,Y,Y')\bigr)
    \right].
$
For every $s\in\mathcal S$, subtracting the last two identities gives
\[
    \mathcal L(\Delta_s)-\mathcal L^*
    =
    \E
    \left[
        \varphi_{\eta(X,Y,Y')}
        \bigl(\Delta_s(X,Y,Y')\bigr)
        -
        \varphi_{\eta(X,Y,Y')}
        \bigl(f^*(X,Y,Y')\bigr)
    \right].
\]
A direct calculation gives
$
    \frac{d}{du}\varphi_\eta(u)
    =
    \sigma(u)-\eta
$
and
$
    \frac{d^2}{du^2}\varphi_\eta(u)
    =
    \sigma(u)(1-\sigma(u))
    \le
    \frac{1}{4}.
$
Thus $\varphi_\eta$ is $1/4$-smooth. Since $f^*(x,y,y')$ is the minimizer of $\varphi_\eta$,
$
    \varphi_\eta(\Delta_s(x,y,y'))
    -
    \varphi_\eta(f^*(x,y,y'))
    \le
    \bigl(\Delta_s(x,y,y')-f^*(x,y,y')\bigr)^2/8 .
$
Taking expectations and then the infimum over $s\in\mathcal S$ proves the upper bound.

For the lower bound, suppose that for every $s\in\mathcal S$,
$\delta<\sigma(\Delta_s(X,Y,Y'))<1-\delta$ and
$\delta<\sigma(f^*(X,Y,Y'))<1-\delta$ almost surely. Since $\sigma$ is monotone, for every
$t$ between $\Delta_s(x,y,y')$ and $f^*(x,y,y')$ we have
$\delta<\sigma(t)<1-\delta$. Hence
$
    \frac{d^2}{dt^2}\varphi_\eta(t)
    =
    \sigma(t)(1-\sigma(t))
    \ge
    \delta(1-\delta)
    \ge
    \delta^2.
$
Because $f^*(x,y,y')$ minimizes $\varphi_\eta$, Taylor's theorem gives
$
    \varphi_\eta(\Delta_s(x,y,y'))
    -
    \varphi_\eta(f^*(x,y,y'))
    \ge
    \frac{\delta^2}{2}
    \bigl(\Delta_s(x,y,y')-f^*(x,y,y')\bigr)^2 .
$
Taking expectations and then the infimum over $s\in\mathcal S$ proves the lower bound.
\end{proof}


\begin{lemma}\label{lem:proj onto score diff}
Let Assumption~\ref{ass:L_2} hold. Among all square-integrable measurable score functions, the best $L_2$ score-difference approximation to $f^*$ is induced by $s^*$. Equivalently, for the joint distribution $P_{X,Y,Y'}$, $\Delta_{s^*}$ is the $L_2(P_{X,Y,Y'})$ projection of $f^*$ onto the space of all score differences. Consequently, for any score class $\mathcal S$,
\begin{align}\label{eq:appro_decom}
    \inf_{s\in\mathcal S}
    \E
    \left[
        \bigl(\Delta_s(X,Y,Y')-f^*(X,Y,Y')\bigr)^2
    \right]
    &=
    \inf_{s\in\mathcal S}
    \E
    \left[
        \bigl(\Delta_s(X,Y,Y')-\Delta_{s^*}(X,Y,Y')\bigr)^2
    \right]
    \notag\\
    &\quad
    +
     \E
    \left[
        \bigl(\Delta_{s^*}(X,Y,Y')-f^*(X,Y,Y')\bigr)^2
    \right].
\end{align}
\end{lemma}
\begin{proof}
Let $R(x,y,y'):=f^*(x,y,y')-\Delta_{s^*}(x,y,y')$. By reciprocity of preferences, $f^*$ is anti-symmetric: $f^*(x,y',y)=-f^*(x,y,y')$.
Since $\Delta_{s^*}$ is also anti-symmetric, $R$ is anti-symmetric. By antisymmetry of $f^*$ in its responses and the definition of $s^*$, we get
$
\mathbb{E}_{Y \sim p(\cdot|x)} s^*(x,Y) = 0,
$
and consequently,
\[
\begin{aligned}
    \E_{U\sim p(\cdot\mid x)}
    R(x,y,U)
    &=
    \E_U f^*(x,y,U)
    -
    s^*(x,y)
    +
    \E_U s^*(x,U)
    =
    0 .
\end{aligned}
\]
By anti-symmetry, $\E_U R(x,U,y)=0$ as well. Now take any square-integrable score $s$. Then
\[
\begin{aligned}
    &\E
    \left[
        R(X,Y,Y')
        \bigl(s(X,Y)-s(X,Y')\bigr)
    \right]
    =
    \E
    \left[
        R(X,Y,Y')s(X,Y)
    \right]
    -
    \E
    \left[
        R(X,Y,Y')s(X,Y')
    \right].
\end{aligned}
\]
Using the exchangeability of $Y$ and $Y'$ conditional on $X$ and the anti-symmetry of $R$, the second term equals the negative of the first. Therefore the previous display equals $2\E[R(X,Y,Y')s(X,Y)]=0$,
where the last equality follows from $\E_{Y'\mid X,Y}R(X,Y,Y')=0$. Hence
$f^*-\Delta_{s^*}$ is orthogonal to every score difference. This proves that
$\Delta_{s^*}$ is the $L_2(P_{X,Y,Y'})$ projection of $f^*$ onto the score-difference subspace.
The Pythagorean identity gives \eqref{eq:appro_decom}.
\end{proof}


\begin{lemma}\label{lem:struc_cycle}
Let Assumption~\ref{ass:L_2} hold, and let $Y_1,Y_2,Y_3$ be conditionally i.i.d. given $X$. Then
\[
    \E
    \left[
        \bigl(\Delta_{s^*}(X,Y,Y')-f^*(X,Y,Y')\bigr)^2
    \right]
    =
    \frac{1}{3}
    \E_{X,Y_1,Y_2,Y_3}
    \left[
    (\mathrm{curl}\, f^*)(X; Y_1,Y_2,Y_3)^2
    \right].
\]
\end{lemma}
\begin{proof}
Let $R(x,y,y'):=f^*(x,y,y')-\Delta_{s^*}(x,y,y')$.
From the proof of Lemma~\ref{lem:proj onto score diff}, $R$ is anti-symmetric and satisfies
$\E_{U\sim p(\cdot\mid x)}R(x,y,U)=0$ and $\E_{U\sim p(\cdot\mid x)}R(x,U,y)=0$.
The combinatorial curl of every score-difference function telescopes to zero. Therefore,
\begin{align}\label{eq:sum_f=sum_r}
    (\mathrm{curl}\, f^*)(X; Y_1,Y_2,Y_3)
    =
    R(X,Y_1,Y_2)+R(X,Y_2,Y_3)+R(X,Y_3,Y_1).
\end{align}
Expanding the square of the right-hand side of \eqref{eq:sum_f=sum_r}, the three squared terms have the same expectation. The cross terms vanish:
\[
\begin{aligned}
    \E
    \bigl[
        R(X,Y_1,Y_2)R(X,Y_2,Y_3)
    \bigr]
    &=
    \E_{X,Y_2}
    \left[
        \E_{Y_1\mid X,Y_2}R(X,Y_1,Y_2)
        \,
        \E_{Y_3\mid X,Y_2}R(X,Y_2,Y_3)
    \right]
    =
    0.
\end{aligned}
\]
The other two cross terms are identical. Hence
\[
    \E
    \left[
        \left(
            R(X,Y_1,Y_2)+R(X,Y_2,Y_3)+R(X,Y_3,Y_1)
        \right)^2
    \right]
    =
    3 \E
    \left[
        \bigl(\Delta_{s^*}(X,Y,Y')-f^*(X,Y,Y')\bigr)^2
    \right].
\]
Combining this identity with \eqref{eq:sum_f=sum_r} proves the claim.
\end{proof}
The proof of Theorem~\ref{thm:triangle-curl bounds} is completed by plugging Lemma~\ref{lem:proj onto score diff} and Lemma~\ref{lem:struc_cycle} into Lemma~\ref{lem:delta-approx}.

\subsection{Proof of Theorem~\ref{thm:score_differ}}\label{apd:thm}
\begin{proof}
Each $\theta\in\Theta(d,L,N)$ induces a score function
$s_\theta\in\mathcal{S}(d,L,N)$, and every
$s\in\mathcal{S}(d,L,N)$ is realized by some
$\theta\in\Theta(d,L,N)$, so the induced triplet-wise hypothesis class satisfies
$
\{f_{\pi_\theta}:\theta\in\Theta(d,L,N)\}
=
\Delta\mathcal{S}(d,L,N).
$
Thus,
$
\inf_{\theta\in\Theta(d,L,N)}\L(f_{\pi_\theta})
=
\inf_{s\in\mathcal{S}(d,L,N)}\L(\Delta_s),
$
and it suffices to work at the level of score functions.

Apply Lemma~\ref{lem:delta-approx} with score class
$\mathcal S(d,\bar L,\bar N)$ followed by Lemma~\ref{lem:proj onto score diff} to obtain
\begin{equation}\label{eq: pol approx bound}
\begin{split}
    \inf_{s\in\mathcal S(d,\bar L,\bar N)}
    \bigl[
        \L(\Delta_s)-\L^*
    \bigr]
    \le
    &\frac{1}{8}
    \inf_{s\in\mathcal S(d,\bar L,\bar N)}
    \E
    \left[
        \bigl(\Delta_s(X,Y,Y')-\Delta_{s^*}(X,Y,Y')\bigr)^2
    \right]\\
    &+\frac{1}{8}
   \E
    \Big[
        \bigl(\Delta_{s^*}(X,Y,Y')-f^*(X,Y,Y')\bigr)^2
    \Big]
\end{split}
\end{equation}
For any score $s$, write $e_s(x,y)=s(x,y)-s^*(x,y)$. Since $Y$ and $Y'$ are conditionally i.i.d. given $X$,
\[
\begin{aligned}
    \mathbb{E}_{X,Y,Y'}
    \left[
        \bigl(\Delta_s(X,Y,Y')-\Delta_{s^*}(X,Y,Y')\bigr)^2
    \right]
    &=
    2\,\mathbb{E}_{X,Y}
    \left[
        e_s(X,Y)^2
    \right]
    -
    2\,\mathbb{E}_{X}
    \left[
        \bigl(\mathbb{E}_{Y\mid X}[e_s(X,Y)]\bigr)^2
    \right]
    \\
    &\le
    2\,\mathbb{E}_{X,Y}
    \left[
        \bigl(s(X,Y)-s^*(X,Y)\bigr)^2
    \right].
\end{aligned}
\]
By Theorem 3.1 in \citet{nagler2026optimal}, there exist constants $C_1,C_{\mathrm{app}}>0$ such that, for all $L_0,N_0\ge2$, there exists $\hat s\in\mathcal S(d,\bar L,\bar N)$ satisfying $\sup_{(x,y)\in\X\times\Y}|\hat s(x,y)-s^*(x,y)|\le C_{\mathrm{app}}(1/(L_0^2N_0^2))^{\beta/d}$.
Thus,
$
    \inf_{s\in\mathcal S(d,\bar L,\bar N)}
    \E
    \left[
        \bigl(\Delta_s(X,Y,Y')-\Delta_{s^*}(X,Y,Y')\bigr)^2
    \right]
    \le
    2C_{\mathrm{app}}^2
    \left(
        1/L_0^2N_0^2
    \right)^{2\beta/d}.
$
Combining this bound and Lemma~\ref{lem:struc_cycle} with (\ref{eq: pol approx bound}) proves the upper bound for the unclipped score class after setting $C_2 = C_{\mathrm{app}}^2/4$ into the constant.
Choosing $L_0\asymp N_0\asymp \epsilon^{-d/(8\beta)}$ makes the first term of order $\epsilon$, and gives the stated $\epsilon$-form of the bound.

For the clipped parameterization, choose
$\hat\theta\in\Theta(d,\bar L,\bar N)$ such that
$s_{\hat\theta}=\hat s$, and set
$\hat s_B=\operatorname{clip}_B\circ\hat s$.
By Assumption~\ref{ass:holder}, $|s^*(x,y)|\le B$,
so clipping gives
\[
|\hat s_B(x,y)-s^*(x,y)|
\le
|\hat s(x,y)-s^*(x,y)|.
\]
Consequently,
$
\sup_{(x,y)\in \mathcal{X} \times \mathcal Y}|\hat s_B(x,y)-s^*(x,y)|
\le
C_{\mathrm{app}}\left(1/(L_0^2N_0^2)\right)^{\beta/d}.
$
The preceding conditional-i.i.d. calculation yields
$
\E\left[
\bigl(\Delta_{\hat s_B}-\Delta_{s^*}\bigr)^2
\right]
\le
2C_{\mathrm{app}}^2
\left(1/(L_0^2N_0^2)\right)^{2\beta/d}.
$
The upper bound statements for the clipped score class then follow in the same way as the unclipped score class argument.

For the lower bound, fix
$0<\delta<\min\{a,\sigma(-2B)\}$.
For every $\theta\in\Theta(d,\bar L,\bar N)$,
$|s_{\theta,B}|\le B$, and hence
$|f_{\pi_{\theta,B}}|\le2B$. Thus,
\[
\delta
<
\sigma(-2B)
\le
\sigma(f_{\pi_{\theta,B}})
\le
\sigma(2B)
=
1-\sigma(-2B)
<
1-\delta.
\]
Moreover,
\[
\delta<a\le\sigma(f^*)=\eta\le1-a<1-\delta
\qquad\text{almost surely}.
\]
The lower-bound assumptions of
Theorem~\ref{thm:triangle-curl bounds}
therefore hold uniformly over the clipped class, proving the stated lower bound.
\end{proof}

\subsection{Proof of Theorem~\ref{thm:generalization}}
Define $\mathfrak R_n(\Delta\mathcal S_B)$
analogously using the inputs $(X_i,Y_i,Y_i')$.
Since $(X,Y)$ and $(X,Y')$ have the same marginal
distribution, symmetry of the Rademacher signs gives
\[
\begin{aligned}
\mathfrak R_n(\Delta\mathcal S_B)
&=
\mathbb E\left[
\sup_{s\in\mathcal S_B}
\frac1n\sum_{i=1}^n
\epsilon_i\bigl(s(X_i,Y_i)-s(X_i,Y_i')\bigr)
\right]\\
&\le
\mathbb E\left[
\sup_{s\in\mathcal S_B}
\frac1n\sum_{i=1}^n\epsilon_i s(X_i,Y_i)
\right] +
\mathbb E\left[
\sup_{s\in\mathcal S_B}
\frac1n\sum_{i=1}^n(-\epsilon_i)s(X_i,Y_i')
\right]\\
&=
2\mathfrak R_n(\mathcal S_B).
\end{aligned}
\]

Let $\ell(u)=\log(1+e^{-u})$, and define the loss class
\[
\mathcal G
:=
\left\{
g_f:(x,y,y',a)\mapsto\ell(af(x,y,y'))\mid
f\in\Delta\mathcal S_B
\right\}.
\]
Every $f\in\Delta\mathcal S_B$ satisfies $|f|\le2B$,
so every $g\in\mathcal G$ takes values in $[0,M]$.
Moreover, $|\ell'(u)|\le1$, and hence $\ell$ is
$1$-Lipschitz. Applying the Rademacher contraction inequality gives
\[
\mathfrak R_n(\mathcal G)
\le
\mathbb E\left[
\sup_{f\in\Delta\mathcal S_B}
\frac1n\sum_{i=1}^n
\epsilon_i A_i f(X_i,Y_i,Y_i')
\right].
\]
Conditional on the observations, the signs
$\epsilon_i A_i$ are independent uniform signs.
Consequently,
\[
\mathfrak R_n(\mathcal G)
\le
\mathfrak R_n(\Delta\mathcal S_B)
\le
2\mathfrak R_n(\mathcal S_B).
\]

Define the one-sided deviations
\[
U_+
:=
\sup_{f\in\Delta\mathcal S_B}
\bigl(\mathcal L(f)-\widehat{\mathcal L}_n(f)\bigr),
\qquad
U_-
:=
\sup_{f\in\Delta\mathcal S_B}
\bigl(\widehat{\mathcal L}_n(f)-\mathcal L(f)\bigr).
\]
Standard symmetrization yields
\[
\mathbb E U_+ \le 2\mathfrak R_n(\mathcal G),
\qquad
\mathbb E U_- \le 2\mathfrak R_n(\mathcal G).
\]
Since $0\le g\le M$, replacing one observation
changes either $U_+$ or $U_-$ by at most $M/n$.
McDiarmid's inequality therefore gives, for $t>0$,
\[
\Pr\bigl(U_\pm>\mathbb E U_\pm+t\bigr)
\le
\exp\left(-\frac{2nt^2}{M^2}\right).
\]
Taking
$t=M\sqrt{\log(2/\alpha)/(2n)}$
and applying a union bound shows that, with
probability at least $1-\alpha$,
\begin{equation}\label{eq: rad}
U_n
:=
\sup_{f\in\Delta\mathcal S_B}
\left|\mathcal L(f)-\widehat{\mathcal L}_n(f)\right|
=
\max\{U_+,U_-\} \le
2\mathfrak R_n(\mathcal G)
+
M\sqrt{\frac{\log(2/\alpha)}{2n}} \le
4\mathfrak R_n(\mathcal S_B)
+
M\sqrt{\frac{\log(2/\alpha)}{2n}}.
\end{equation}

On this event, the empirical minimization property
implies that, for every
$\theta\in\Theta(d,\bar L,\bar N)$,
\[
\begin{aligned}
\mathcal L(f_{\pi_{\widehat\theta,B}})
&\le
\widehat{\mathcal L}_n(f_{\pi_{\widehat\theta,B}})
+U_n \le
\widehat{\mathcal L}_n(f_{\pi_{\theta,B}})
+U_n \le
\mathcal L(f_{\pi_{\theta,B}})+2U_n.
\end{aligned}
\]
Taking the infimum over $\theta$ gives
\[
\begin{aligned}
\mathcal L(f_{\pi_{\widehat\theta,B}})-\mathcal L^*
\le{}&
\inf_{\theta\in\Theta(d,\bar L,\bar N)}
\bigl[\mathcal L(f_{\pi_{\theta,B}})-\mathcal L^*\bigr]
+2U_n.
\end{aligned}
\]
Combining this inequality with the clipped upper
bound in Theorem~\ref{thm:score_differ} and equation~(\ref{eq: rad}) proves
the asserted excess-risk bound.

Finally, Theorem~\ref{thm:dpo-calibration-delta}
implies
\[
P^*-P(f_{\pi_{\widehat\theta,B}})
\le
\sqrt{2}
\left[
\mathcal L(f_{\pi_{\widehat\theta,B}})
-\mathcal L^*
\right]^{1/2},
\]
which proves the CPA bound on the same event.

\end{document}